\pdfoutput=1

\documentclass[11pt]{article}

\usepackage{acl}

\usepackage{times}
\usepackage{latexsym}

\usepackage[T1]{fontenc}

\usepackage[utf8]{inputenc}

\usepackage{microtype}

\usepackage{inconsolata}

\usepackage{amsmath}
\usepackage{amsfonts}
\usepackage{amssymb}
\usepackage{amsthm}
\usepackage{bbold}
\usepackage{mathtools}

\usepackage{graphicx}

\newtheorem{prop}{Proposition}

\newcommand{\N}[0]{\mathbb{N}}
\newcommand{\R}[0]{\mathbb{R}}

\newcommand{\E}[0]{\mathbb{E}}
\newcommand{\M}[0]{\mathbb{M}}

\newcommand{\bx}[0]{\pmb{x}}
\newcommand{\bX}[0]{\pmb{X}}
\newcommand{\bbX}[0]{\mathbb{X}}
\newcommand{\cX}[0]{\mathcal{X}}

\newcommand{\bbY}[0]{\mathbb{Y}}
\newcommand{\cY}[0]{\mathcal{Y}}

\title{
Show Your Work with Confidence: \\
Confidence Bands for Tuning Curves
}

\author{
  Nicholas Lourie \\
  New York University \\
  \texttt{nick.lourie@nyu.edu} \\\And
  Kyunghyun Cho \\
  New York University \& Genentech \\
  \texttt{kyunghyun.cho@nyu.edu} \\\And 
  He He \\
  New York University \\
  \texttt{hhe@nyu.edu} \\
}

\begin{document}

\maketitle

\begin{abstract}
    The choice of hyperparameters greatly impacts performance in natural language processing.
    Often, it is hard to tell if a method is better than another or just better tuned.
    \textit{Tuning curves} fix this ambiguity by accounting for tuning effort. Specifically, they plot validation performance as a function of the number of hyperparameter choices tried so far.
    While several estimators exist for these curves, it is common to use point estimates, which we show fail silently and give contradictory results when given too little data.
    
    Beyond point estimates, \textit{confidence bands} are necessary to rigorously establish the relationship between different approaches.
    We present the first method to construct valid confidence bands for tuning curves. The bands are exact, simultaneous, and distribution-free, thus they provide a robust basis for comparing methods.

    Empirical analysis shows that while bootstrap confidence bands, which serve as a baseline, fail to approximate their target confidence, ours achieve it exactly.
    We validate our design with ablations, analyze the effect of sample size, and provide guidance on comparing models with our method.
    To promote confident comparisons in future work, we release \texttt{opda}: an easy-to-use library that you can install with \texttt{pip}. \url{https://github.com/nicholaslourie/opda}
\end{abstract}

\section{Introduction}
\label{sec:introduction}

\begin{figure}[t]
    \centering
    \includegraphics[width=\linewidth]{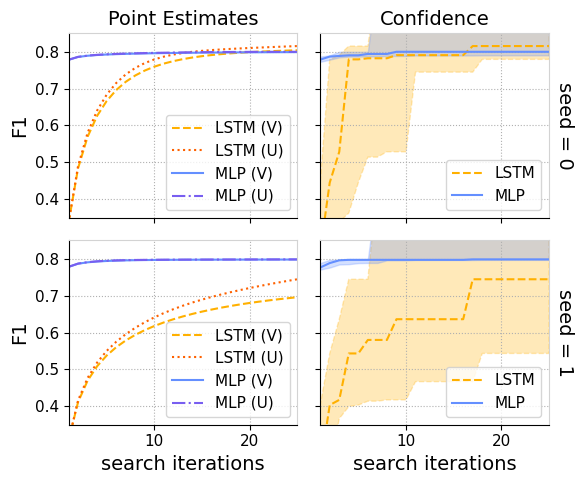}
    \caption{
        The charts compare tuning curves for MLP and LSTM text classifiers on Reuters \citep{apte-etal-1994-automated}, based on data from \citet{tang-etal-2020-showing}. The \textit{tuning curve} plots the F1 score of the best model after each round of random search. The left compares point estimates for the mean based on U and V-statistics; the right compares 50\% confidence bands for the median tuning curve. The top and the bottom run the same analysis on different samples of 25 search iterations. \textit{The point estimates give contradictory conclusions without warning on different samples, disagreeing whether the LSTM ever beats the MLP. The differences between estimators are small in comparison to this sample variation. Confidence bands, in contrast, directly show the variation due to sampling.}
    }
    \label{fig:comparing-with-point-estimates-vs-confidence}
\end{figure}

Accounting for hyperparameters when comparing models is an important, open problem in NLP and deep learning research. This problem is particularly relevant right now, as the rush to scale up has left us with large, costly language models and little understanding of how different designs compare. Indeed, many earlier models are now outperformed by others an order of magnitude smaller---the main difference: better hyperparameters. Even worse, the challenge of managing hyperparameters during research has produced false scientific conclusions, such as the belief that model size should scale faster than data \citep{hoffman2022training}. As a scientific community, we require more rigorous and reliable analyses for understanding if a model is well-tuned and how costly that process is.

To tackle this issue, researchers have developed the tuning curve \citep{dodge-etal-2019-show, tang-etal-2020-showing, dodge-etal-2021-expected-validation}. The \textit{tuning curve} plots the current best validation score as a function of tuning effort, for example the search iterations or total compute used in hyperparameter tuning (see Figure~\ref{fig:comparing-with-point-estimates-vs-confidence}). By comparing tuning curves, you can determine which method is best for a given budget or tell if the hyperparameters are fully tuned. 

In theory, tuning curves fix the hyperparameter comparison problem; in practice, these curves must be estimated from data. Prior work has developed efficient estimators for tuning curves \citep{dodge-etal-2019-show}; however, these estimators lack corresponding methods to quantify their uncertainty. Generic techniques, such as bootstrap resampling \citep{efron1994introduction}, break down for tuning curve estimates and thus fail to meaningfully capture their uncertainty \citep{tang-etal-2020-showing}. As a result, it is hard to know whether a conclusion is trustworthy or if more data is required (\S\ref{sec:analysis:comparisons-to-existing-methods}). The absence of confidence bands has hindered tuning curves' more widespread adoption. It is this gap in the literature we seek to fill.

To address these shortcomings, we present the first confidence bands for tuning curves. Where the bootstrap breaks down, these confidence bands achieve meaningful coverage---both empirically and provably. Beyond this basic necessity, they also have desirable theoretical properties; namely, they are \textit{exact}, \textit{simultaneous}, and \textit{distribution-free}.

Being \textit{exact}, the bands contain the true tuning curve with precisely the prescribed probability---even in finite samples. Thus, at 95\% confidence, they have 95\% \textit{coverage}, or contain the curve 95\% of the time. In contrast, other confidence bands are often \textit{conservative}, containing the true curve more often than claimed, or \textit{asymptotic}, attaining correct coverage only as the sample size goes to infinity.

Being \textit{simultaneous}, the bands contain the entire tuning curve at once. This guarantee is markedly stronger than \textit{pointwise} coverage, or covering each point separately with the desired probability. For example, pointwise 95\% confidence bands would often fail to cover a part of the true curve, while simultaneous bands will rarely fail to do so.

And being \textit{distribution-free}, the bands are free from restrictive parametric assumptions. As long as the validation score has a continuous distribution, the bands are exact; and even if the distribution is not continuous, the bands will still be conservative. As a result, the bands are guaranteed under general assumptions and thus widely applicable.

Thus, the confidence bands provide a scientific basis for evaluating models. The tuning curve captures tuning effort, while the confidence bands quantify uncertainty in the conclusion. Being exact, that quantification is \textit{precise}; being simultaneous, we can assess the model \textit{across all tuning budgets}; and being distribution-free, the results are \textit{reliable and robust}.

Our key insight is: take simultaneous confidence bands for the test scores' cumulative distribution function (CDF), then translate them to the tuning curve via an algebraic relationship (\S\ref{sec:theory:bounding-the-tuning-curve}). Namely, we start with nonparametric bounds for the CDF, translate those to bound the CDF of the best score for a given tuning budget, then leverage that to bound a summary of the best score's distribution (e.g., the median or the mean). Figure~\ref{fig:deberta-vs-debertav3} exhibits the end result: given competing models (DeBERTa and DeBERTaV3), we account for tuning effort by comparing tuning curves, and for sample variation by reporting them with confidence bands. In this way, the confidence bands empower researchers and practitioners to confidently identify the cost regimes where one method outperforms the other.

Complementing theoretical analysis, we study the confidence bands empirically. Even when point estimates agree, different samples give different results due to undetected sample variation---which the bootstrap fails to capture (\S\ref{sec:analysis:comparisons-to-existing-methods}). In contrast, our confidence bands attain exact coverage, both in theory and practice, as we confirm in experiments (\S\ref{sec:analysis:exact-coverage}). Beyond its core strategy, our approach also incorporates several more subtle design decisions. We show via ablation studies how each tightens the confidence bands  (\S\ref{sec:analysis:ablations}). Generalizing prior work, we then consider \textit{median} as well as \textit{mean} tuning curves, and find the median provides a more useful, interpretable, and tractable point of comparison (\S\ref{sec:analysis:mean-vs-median-tuning-curves}). Last, while confidence bands reveal sample variation, they do not eliminate it---so, we study the effect of sample size to guide how much data is necessary to estimate the tuning curve well (\S\ref{sec:analysis:effect-of-sample-size}).

Before diving into details, let us exhibit the end result. The next section demonstrates how to use the confidence bands in a practical scenario drawn from the recent literature: evaluating DeBERTaV3 \citep{he2021debertav3} against its baseline, DeBERTa \citep{he2021deberta} (\S\ref{sec:tutorial-evaluating-debertav3}). Bringing it all together, this tutorial shows how tuning curves promote reliable comparisons by accounting for tuning effort, while confidence bands reveal when a method is truly better versus when more data is required to reach a conclusion. Because our confidence bands are distribution-free with exact coverage, they provide a rigorous, statistical basis for comparing methods that involve hyperparameters, sampling, or random initialization. To promote reliable comparisons and more reproducible research, we release an easy-to-use library implementing our confidence bands at \url{https://github.com/nicholaslourie/opda}.

\begin{figure}[t]
    \centering
    \includegraphics[width=\linewidth]{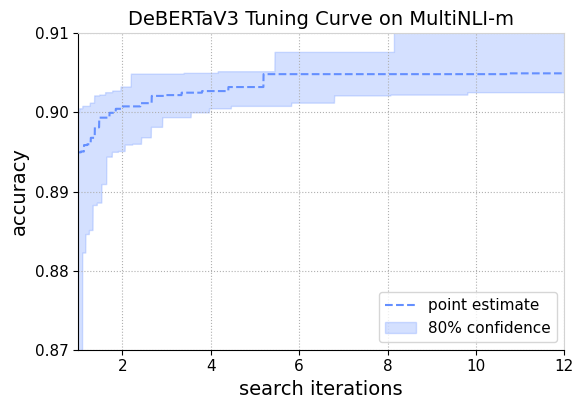}
    \caption{
        The median tuning curve for DeBERTaV3 on MultiNLI (matched), based on 48 search iterations. The point estimate plots the empirical CDF's tuning curve.
    }
    \label{fig:debertav3-tuning-curve}
\end{figure}

\section{Tutorial: Evaluating DeBERTaV3}
\label{sec:tutorial-evaluating-debertav3}

Let's walk through a case study fine-tuning a pretrained model, DeBERTaV3 \citep{he2021debertav3}, to see how our analysis works in practice.

First, we design the search distribution. Just like grid search, for each hyperparameter we choose a linear or log scale, pick upper and lower bounds, then set a (log) uniform distribution between them. For DeBERTaV3, we use a log scale for the learning rate (1e-6, 1e-3), and a linear scale for number of epochs (1, 4), batch size (16, 64), proportion of the first epoch for learning rate warmup (0, 0.6), and dropout (0, 0.3). With this distribution, we then run 48 rounds of random search.

From these samples, we construct the tuning curve and confidence bands in Figure~\ref{fig:debertav3-tuning-curve}. This plot quantifies the model's performance as a function of tuning budget. From it, practitioners can tell if the model is cost effective, and researchers can judge whether it is fully tuned. We see that DeBERTaV3 has high accuracy after only 2-6 rounds of search.

Beyond absolute judgments, we need relative comparisons. Figure~\ref{fig:deberta-vs-debertav3} plots DeBERTaV3 against its baseline, DeBERTa.\footnote{
DeBERTaV2 was only informally released, thus \citet{he2021debertav3} did not compare against it.
}
It is tempting to require that the bands have no overlap before deciding one model is better than the other; however, this rule is too conservative. Inspired by \citet{minka2002judging}, we (tentatively) suggest a heuristic: evidence is \textit{weak} if one band excludes the other's point estimate, \textit{fair} if each band excludes the other's point estimate, and \textit{strong} if the bands have no overlap, for a nontrivial portion of the curve. Thus, there is strong evidence DeBERTaV3 beats DeBERTa for all budgets.

After establishing DeBERTaV3 beats DeBERTa, we might ask: what makes the model work so well? We could test if a hyperparameter is important to tune, or ablate components to see what improves the score the most. Since ablations work like model comparisons, let us analyze a hyperparameter.

Hyperparameter importance can be defined in many ways \citep{hutter-etal-2014-efficient, van-rijn-etal-2018-hyperparameter, probst-etal-2019-tunability}. \citet{weerts2020importance} give an intuitive definition as \textit{tuning risk}: the difference in performance between tuning all hyperparameters and leaving one at its default. Tuning risk depends on tuning effort; let's measure it with tuning curves. Figure~\ref{fig:epoch-importance} compares tuning all hyperparameters to fixing epochs at its default. As epochs impacts the training time, we multiply the search iterations by an appropriate measure of cost (average epochs per iteration). Adjusted for compute, we find weak evidence that the default wins over tuning epochs.

\begin{figure}[t]
    \centering
    \includegraphics[width=\linewidth]{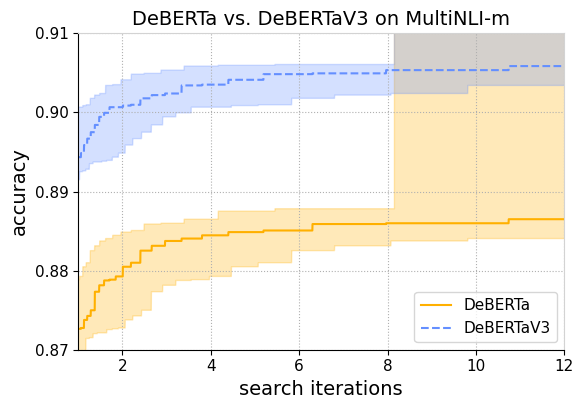}
    \caption{
        Median tuning curves on MultiNLI (matched), with 80\% confidence based on 48 search iterations. Point estimates plot the empirical CDFs' tuning curves.
    }
    \label{fig:deberta-vs-debertav3}
\end{figure}

\begin{figure}[t]
    \centering
    \includegraphics[width=\linewidth]{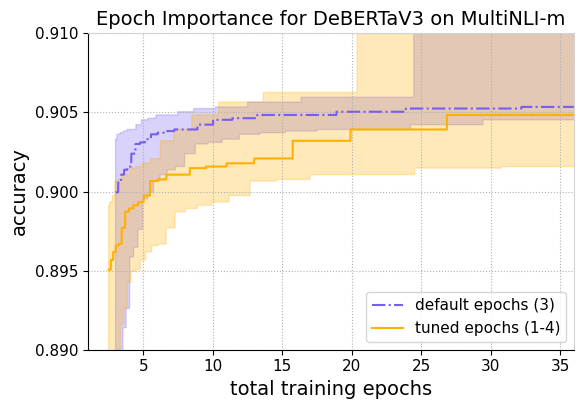}
    \caption{
        Median tuning curves on MultiNLI (matched), with 80\% confidence based on 48 search iterations. To assess hyperparameter importance, the curves compare tuning epochs (1-4) against leaving it at the default (3).
    }
    \label{fig:epoch-importance}
\end{figure}

\section{Theory}
\label{sec:theory}

We derive confidence bands for tuning curves that are simultaneous and distribution-free. These bands are conservative for the mean tuning curve and exact for the median tuning curve when the score distribution is continuous. See \S\ref{app:proofs} for proofs.

\subsection{Formalizing the Problem}
\label{sec:theory:formalizing-the-problem}

Almost every NLP method comes with a number of hyperparameters. These hyperparameters can take any kind of value: categorical, ordinal, complex. Together, they define the \textit{hyperparameter search space}, $\bx = [x_1, \ldots, x_d] \in \bbX$. Given a choice of hyperparameters, $\bx$, evaluating the method yields a real-valued score, $y \in \bbY \subset \R$, like accuracy.

Tuning algorithms search for the best possible hyperparameters by evaluating many choices. In general, the choices, $\bX_1, \ldots, \bX_n$, and their scores, $Y_1, \ldots, Y_n$, are random variables. Almost all of the search's expense comes from training models, so the number of evaluations, $n$, is a good proxy for cost. When comparing models of different sizes, this number should be multiplied by the average cost per evaluation (for example, in FLOPs), to ensure a fair comparison. In this way, the cost-benefit trade-off between tuning effort and task performance is captured by the \textit{tuning process}:
\begin{equation}\label{eq:tuning-process}
    T_k \coloneqq \max_{i=1\ldots k} Y_i
\end{equation}
or, the best performance after each evaluation.

Random search samples choices independently from a user provided \textit{search distribution}, $\bX_i \sim \cX$. This distribution over choices induces the \textit{score distribution} over the test metric: $Y_i \sim \cY$. Since every round's score is independent and identically distributed, each $T_k$ is just the max from a sample of size $k$. So, if $F(y) = P(Y \leq y)$ is the CDF for $Y$, and $F_k(y) = P(T_k \leq y)$ is the CDF for $T_k$, then $P\left(\max_{i=1\dots k} Y_i \leq y\right) = P(Y \leq y)^k$ implies:
\begin{equation}\label{eq:cdf-of-max}
    F_k(y) = F(y)^k
\end{equation}
Using Equation~\ref{eq:cdf-of-max}, we extend the definition of $T_k$ to all positive real numbers $k \in \R_{>0}$ (rather than just natural numbers, $k \in \N$), by letting $T_k$ be the random variable with CDF $F(y)^k$.\footnote{
Note this definition loses the joint distribution, matching Equation~\ref{eq:tuning-process} only in the marginal distributions for each $T_k$.
}
We can now define the \textit{expected} tuning curve, $\tau_e(k) \coloneqq \E[T_k]$. While prior work focuses on the expected tuning curve, the \textit{median} tuning curve, $\tau_m(k) \coloneqq \M[T_k]$, has several advantages that we will explore (\S\ref{sec:analysis:mean-vs-median-tuning-curves}).

\subsection{Bounding the Tuning Curve}
\label{sec:theory:bounding-the-tuning-curve}

Our core strategy translates bounds on one round of random search into bounds on the best of $k$ rounds. \S\ref{sec:theory:bounding-the-cdf} will describe how to bound one round's CDF, $\widehat{F}^l(y) \leq F(y) \leq \widehat{F}^u(y)$, using the order statistics, $Y_{(1)}, \ldots, Y_{(n)}$ (where $Y_{(i)}$ is the $i$-th least value). Equation~\ref{eq:cdf-of-max} then translates these CDF bounds on $Y$ into CDF bounds on $T_k$:
\begin{equation}\label{eq:tuning-process-cdf-bounds}
    (\widehat{F}^l(y))^k \leq F_k(y) \leq (\widehat{F}^u(y))^k    
\end{equation}
As our method bounds the \textit{entire} CDF, it describes the full \textit{distribution} of outcomes from random search. By taking the mean or median, we can translate the upper and lower CDF bands into lower and upper confidence bands for the tuning curve.

\paragraph{Mean Tuning Curves.} 
The confidence bands for the CDF explored in \S\ref{sec:theory:bounding-the-cdf} place all probability mass on the order statistics and the bounds on $Y$. In this case, we can convert a CDF, $F_k$, into its expectation by summing over the order statistics and bounds weighted by their probability:
\begin{equation}\label{eq:cdf-to-expected-tuning-curve}
    \hat{\tau}_e(k) = \sum_{i=0}^{n+1} Y_{(i)} \left[\widehat{F}\left(Y_{(i)}\right)^k - \widehat{F}\left(Y_{(i-1)}\right)^k\right]
\end{equation}
where $Y_{(n+1)}$ is the upper bound on $Y$, $Y_{(0)}$ is the lower bound, and $F\left(Y_{(-1)}\right)$ is $0$ by convention. It is possible that $Y_{(n+1)} = \infty$, $Y_{(0)} = -\infty$, or both in which cases the mean could be $\infty$, $-\infty$, or undefined. If $Y$ has no finite upper bound, then the upper confidence band for the tuning curve will be vacuous. Similarly, if $Y$ has no finite lower bound, the lower confidence band will be vacuous as well.

\paragraph{Median Tuning Curves.}
In contrast, the median curve does not require finite bounds on $Y$ to yield meaningful confidence bands. We convert the CDF, $F_k$, into the median tuning curve directly by:
\begin{equation}\label{eq:cdf-to-median-tuning-curve}
    \hat{\tau}_m(k) = \min \left\{Y_{(i)} \mid 0.5 \leq \widehat{F}\left(Y_{(i)}\right)^k \right\}
\end{equation}
If $Y$ has no finite bounds, then the upper and lower confidence limits will take finite values until the bands' endpoints, where they will diverge to $\pm \infty$. Since the CDF confidence bands are discrete, the median is generally ambiguous. Our definition takes the minimum value above at least 50\% of the probability mass. This choice is motivated by the classic definition of the quantile function, $Q(p) \coloneqq \inf \{y \in \bbY \mid p \leq F(y) \}$, which makes it a left inverse of the CDF everywhere that the distribution assigns nonzero probability mass, so $Q(F(Y)) = Y$ holds with probability one.

\subsection{Bounding the CDF}
\label{sec:theory:bounding-the-cdf}

Our method for bounding the tuning curve (\S\ref{sec:theory:bounding-the-tuning-curve}) requires simultaneous CDF bands as input. The literature offers several ways to construct these.

Given $n$ independent random variables, $Y_i \sim \cY$, we can approximate their CDF via the \textit{empirical cumulative distribution function} (eCDF):
\begin{equation}\label{eq:ecdf}
    \widehat{F}(y) \coloneqq \frac{1}{n} \sum_{i=1}^{n} \mathbb{1}[Y_i \leq y]
\end{equation}
By the Glivenko-Cantelli theorem, this converges uniformly, almost surely to the CDF \citep{vaart_1998}.

\paragraph{DKW Bands.} 
The Dvoretzky-Kiefer-Wolfowitz (\textit{DKW}) inequality characterizes this convergence rate \citep{dvoretzky-etal-1956-asymptotic, massart-1990-tight}, $\epsilon > 0$:
\begin{equation}\label{eq:dkw-inequality}
    P\left(\sup_{y \in \bbY} \left|\widehat{F}(y) - F(y)\right| > \epsilon \right) \leq 2e^{-2n\epsilon^2}
\end{equation}
In particular, the inequality bounds how far the eCDF is from the true CDF with high probability.

Setting the left to $1-\alpha$ then solving for $\epsilon$ yields simultaneous $1-\alpha$ confidence bands for the CDF:
\begin{equation}\label{eq:dkw-confidence-bands}
    \widehat{F}(y) - \epsilon \leq F(y) \leq \widehat{F}(y) + \epsilon; \hspace{0.5em} \epsilon \coloneqq \sqrt{\frac{\log \frac{2}{\alpha}}{2n}}
\end{equation}
These DKW bands (and the inequality) are tight asymptotically, but conservative in finite samples.

\paragraph{KS Bands.}
The Kolmogorov-Smirnov (\textit{KS}) test tightens the DKW inequality for finite samples by taking the supremum as a test statistic:
\begin{equation}\label{eq:ks-statistic}
   D_n \coloneqq \sup_{y \in \bbY} \left|\widehat{F}(y) - F(y)\right|
\end{equation}
If $F$ is the true CDF and continuous, then $D_n$'s distribution does not depend on $Y$'s. In that case, $D_n$ has a (two-sided) KS distribution, and the KS test is exact and distribution-free \citep{bradley1968distribution}. We can invert it to construct simultaneous, exact, and distribution-free confidence bands for the CDF.

The DKW inequality and KS test are central tools in nonparametric statistics. The first offers a simple, closed-form formula, while the second is tighter, quick to compute, and widely implemented in software. Still, both methods share a major shortcoming: because the bands have constant width over the CDF, they are violated more often near the median. As a result, they are wider than necessary, especially at the extremes---and it is the extremes which most impact the tuning curve.

\paragraph{LD Bands.}
To fix this issue, \citet{learned-miller-etal-2008-probabilistic} derived simultaneous confidence bands that are violated equally often at all points, and thus are much narrower at the extremes. The Learned-Miller-DeStefano (LD) bands are based on the \textit{order statistics}: $Y_{(1)}, \ldots, Y_{(n)}$, where $Y_{(i)}$ is the sample's $i$-th smallest number. The basic idea behind the derivation is: first bound the CDF at the order statistics, then extend those bounds to the rest of the function.

Let us first bound the CDF at the order statistics. Consider a random variable $Y$. If $Y$ is continuously distributed and $F$ its CDF, then $F(Y)$ is uniformly distributed between $0$ to $1$. Since the CDF is always increasing and thus order preserving, $F\left(Y_{(i)}\right)$ is the $i$-th order statistic of $F\left(Y_{(1)}\right), \ldots, F\left(Y_{(n)}\right)$. As a result, $F\left(Y_{(i)}\right)$ is distributed as the uniform's $i$-th order statistic, or $\text{Beta}(i, n{+}1{-}i)$. Then any interval containing $1-\alpha'$ of the Beta distribution's probability mass is a $1-\alpha'$ confidence interval for the CDF at $Y_{(i)}$. Let the limits of these intervals be $l_{(i)} \leq F\left(Y_{(i)}\right) \leq u_{(i)}$. The LD bands set $\alpha'$ so that these pointwise intervals hold simultaneously with probability $1-\alpha$.

Once we have bounds at the order statistics, we can extend them across the rest of the CDF. Since the CDF is monotonic, the lower bound on $F\left(Y_{(i)}\right)$ extends to the right and the upper bound on $F\left(Y_{(i)}\right)$ extends to the left. Formally, $Y_{(i)} < y$ implies that $l_{(i)} \leq F\left(Y_{(i)}\right) \leq F(y)$ and $y < Y_{(i)}$ implies that $F(y) \leq F\left(Y_{(i)}\right) \leq u_{(i)}$. In this way, the confidence intervals at the order statistics bound the \textit{entire} CDF.

The LD bands are not widely implemented---perhaps in part because they are hard to compute. The main difficulty is finding $\alpha'$. \citet{learned-miller-etal-2008-probabilistic} recommend adjusting $\alpha'$ (e.g., using binary search) until you have obtained the desired coverage. This coverage is measured in simulation. Since each iteration constructs the confidence bands thousands of times, this approach ends up being very time consuming. Instead, we reformulate the confidence bands as a hypothesis test and directly simulate the test statistic's null distribution.

Specifically, we define the test statistic:
\begin{equation}\label{eq:l-statistic}
    L_n \coloneqq \max_{i=1\ldots n} B_{i}\left(F\left(Y_{(i)}\right)\right)
\end{equation}
where $B_i(p)$ is the coverage under $\text{Beta}(i, n{+}1{-}i)$ of the smallest interval containing $p$. In general, you could consider different types of intervals. The \textit{highest probability density interval} produces the best confidence bounds (\S\ref{sec:analysis:ablations}), though the \textit{equal-tailed interval} is easier to compute.\footnote{
For algorithms to compute $B_i(p)$ for both, see \S\ref{app:additional-algorithms}.
}
If $\cY$ is a continuous distribution, then $F\left(Y\right)$ is uniformly distributed between $0$ and $1$ and $L_n$ always has the same distribution. We compute significance levels by simulating this distribution using the uniform.

Given this test, we invert it to create confidence bands: First, find the $1-\alpha$ quantile for $L_n$, next use it to create confidence intervals for each $F\left(Y_{(i)}\right)$, and finally extend these intervals via monotonicity, as before. This approach requires only one round of simulation, rather than one for each step of a binary search---leading to a substantial speed up.

\section{Experimental Setup}
\label{sec:experimental-setup}

Our experiments address two kinds of questions:

\paragraph{Comparisons to Prior Work.} 
We adapt data\footnote{
\url{github.com/castorini/meanmax} (commit: 0ea1241)
} 
from \citet{tang-etal-2020-showing}, who identified many of the challenges inspiring our work. The data consists of 145 and 152 rounds of random search for MLP and LSTM classifiers' F1 scores on the Reuters document classification dataset \citep{apte-etal-1994-automated}. The hyperparameter search included learning rate, batch size, and dropout among others. Like \citet{tang-etal-2020-showing}, we measure the empirical coverage of bootstrap confidence bands in simulation. We fit kernel density estimates (KDE) to MLP and LSTM hyperparameter tuning data, to simulate random search while also being able to compute the true tuning curve in this simulation. We improve \citeposs{tang-etal-2020-showing} protocol by using a KDE method that includes a boundary correction, to account for F1 being bounded by 0 and 1. See \S\ref{app:experimental-setup:comparisons-to-prior-work} for details.

\paragraph{Analysis in a Realistic Scenario} 
In addition to the experiments above, we demonstrate our tool on a practical problem from the recent literature: comparing DeBERTaV3 \citep{he2021debertav3} to its baseline, DeBERTa \citep{he2021deberta}. Using the original implementation,\footnote{
\url{github.com/microsoft/DeBERTa} (commit: c558ad9)
} 
we trained DeBERTa and DeBERTaV3 on MultiNLI \citep{williams-etal-2018-broad} evaluated with accuracy over 1,024 rounds of random search. We tuned the dropout, learning rate, batch size, number of epochs, and warmup proportion of the first epoch for both models. Last of all, we fit a KDE model to this data for our exact coverage analysis (\S\ref{sec:analysis:exact-coverage}). See \S\ref{app:experimental-setup:analysis-in-a-realistic-nlp-scenario} for more details.

\section{Analysis}
\label{sec:analysis}

We validate our theory with empirical results.

\subsection{Comparisons to Existing Methods}
\label{sec:analysis:comparisons-to-existing-methods}

First, let us consider current practice. What tools are available to manage hyperparameters during research? Most commonly, researchers just tune all hyperparameters before making a comparison; however, efficient estimators for the tuning curve are already available. Still, prior work raises an important issue with these point estimates: when there is too little data, they give false conclusions silently \citep{tang-etal-2020-showing, dodge-etal-2021-expected-validation}. Unfortunately, it can be difficult to know when the data is too little. In principle, confidence bands could resolve this issue by warning when more data is required; however, looking at within sample variation does not help, since bootstrap confidence bands fail to achieve meaningful coverage \citep{tang-etal-2020-showing}. In contrast, our confidence bands come with strong theoretical guarantees. Let us see how they resolve these issues in practice.

\paragraph{Point Estimators' Drawbacks.} 
Experimental conclusions should not depend on the choice of estimator. Thus, if two estimators would disagree, the experiment should be redesigned. Perhaps the treatments should be changed to produce a greater effect, or the experiment needs a larger sample size.

Figure~\ref{fig:comparing-with-point-estimates-vs-confidence} exemplifies such a situation. In it, the left column shows point estimates, while the right shows confidence bands; and the top row uses one sample, while the bottom uses another. We see two point estimates disagree (top left): the V-statistic estimate \citep{dodge-etal-2019-show} and the U-statistic estimate \citep{tang-etal-2020-showing}. The V-statistic claims the LSTM outperforms the MLP after 20 iterations of random search. On the other hand, the U-statistic claims the LSTM outperforms after only 13, almost half. Even worse than this disagreement, running the same analysis on a second sample gives a totally different result (bottom left): the LSTM never beats the MLP at all.

In both cases, the point estimates disagree on details but reach the same general conclusion. In one, the LSTM eventually wins; in the other, it does not. Either way, we might not have questioned the results. In contrast, confidence bands warn of the contradiction and explain these disagreements (right): there just is not enough data for an answer. While point estimates fail silently, the confidence bands reveal issues in the experiment's design.

\paragraph{Shortcomings of the Bootstrap.}
Bootstrapping often offers an easy way to construct confidence bands \citep{efron1994introduction}; unfortunately, it fails for tuning curves. Typically, the bootstrap yields pointwise confidence bands that achieve the correct coverage asymptotically as the sample size goes to infinity. Nonetheless, while bootstrapping works under mild assumptions, \citet{tang-etal-2020-showing} showed tuning curves do not satisfy them.

Figure~\ref{fig:bootstrap-confidence-band-coverage} reproduces and expands their result, demonstrating it for the U-statistic as well as the V-statistic estimator. Each point along the X-axis corresponds to a point along the tuning curve. The Y-axis represents that point's (empirical) coverage, or percent of the time the bootstrap confidence bands contained it in simulations. At the start of the curve, the bootstrap bands get close to the desired coverage, but as the search iterations increase the coverage plummets.

The bootstrap bands and our confidence bands are not directly comparable because our bands aim for simultaneous rather than pointwise coverage. Still, while the bootstrap bands break down, ours live up to their guarantees and attain exactly the specified level of coverage, as we will see in \S\ref{sec:analysis:exact-coverage}.

\begin{figure}[t]
    \centering
    \includegraphics[width=\linewidth]{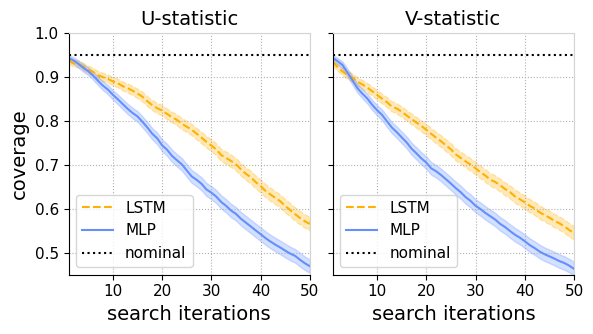}
    \caption{The pointwise coverage of 95\% bootstrap confidence bands constructed from 50 search iterations. The graphs show the coverage at each point of the tuning curve, measured in simulation. The shaded regions are 95\% Clopper-Pearson confidence intervals.}
    \label{fig:bootstrap-confidence-band-coverage}
\end{figure}

\subsection{Exact Coverage}
\label{sec:analysis:exact-coverage}

To test our theoretical guarantees, we measure the coverage empirically. The main challenge is that we never know the true tuning curve. Thus, we can never tell if the confidence bands actually cover it. To overcome this obstacle, we adapt the protocol from \citet{tang-etal-2020-showing}. Specifically, we construct a realistic simulation where we do know the true tuning curve, then use it to validate the coverage.

First, we run 1,024 rounds of random search on DeBERTa and DeBERTaV3. Next, we fit kernel density estimates to the scores from random search (as described in \S\ref{sec:experimental-setup}).\footnote{
KDE must be used instead of bootstrapping to avoid artifacts caused by resampling's discreteness (See \S\ref{app:experimental-setup} for details).
}
These kernel density estimates enable us to simulate each model's score distribution while knowing the true tuning curve exactly. Using the simulations, we then construct the confidence bands, check if they contain the true tuning curve, and record how often they do. Finally, we repeat this process 1,024 times to estimate the empirical coverage, or percent of the time the confidence bands contain the true curve.

Figure~\ref{fig:exact-tuning-curve-coverage} summarizes the results. The X-axis gives the desired confidence level, while the Y-axis shows the actual coverage. As theory predicts, the LD bands achieve exact coverage for the median tuning curve across all models and validation sets.

\begin{figure}[t]
    \centering
    \includegraphics[width=\linewidth]{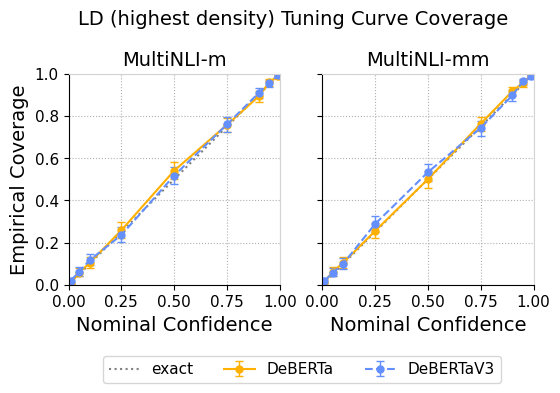}
    \caption{
        Nominal vs. empirical coverage of the LD median tuning curve confidence bands. Error bars show 99\% Clopper-Pearson confidence intervals. The results lie on the $y=x$ line, indicating exact coverage.
    }
    \label{fig:exact-tuning-curve-coverage}
\end{figure}

\subsection{Ablations}
\label{sec:analysis:ablations}

Our basic strategy converts \textit{any} simultaneous CDF bands into bounds on the tuning curve. The tuning curve bands' tightness depends heavily on the CDF bands' shape. Thus, the choice of CDF bands is an important design decision. Our recommended choice, the \textit{highest density LD bands}, gives a tighter bound over a greater range than the alternatives.

\begin{figure*}[t]
    \centering
    \begin{minipage}{0.49\textwidth}
        \centering
        \includegraphics[width=\linewidth]{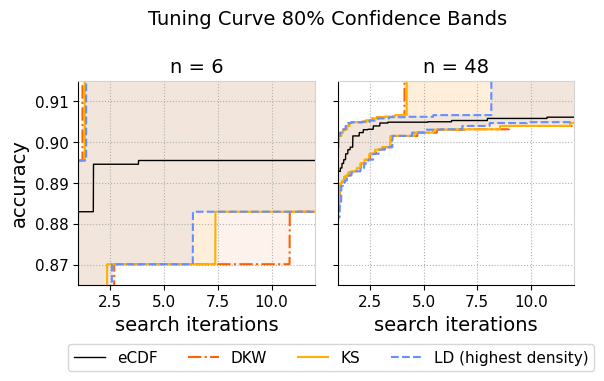}
        \caption{
            The DKW, KS, and LD confidence bands for DeBERTaV3's tuning curve on MultiNLI (matched).
        }
        \label{fig:dkw-vs-ks-vs-ld-tuning-curve-bands_small}
    \end{minipage}\hfill
    \begin{minipage}{0.49\textwidth}
        \centering
        \includegraphics[width=\linewidth]{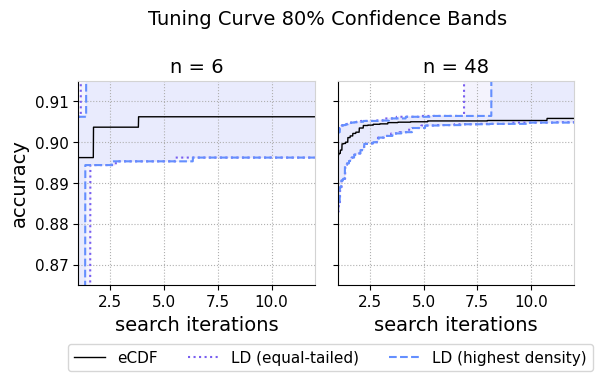}
        \caption{
            Equal-tailed and highest density LD bands for DeBERTaV3's tuning curve on MultiNLI (matched).
        }
        \label{fig:equal-tailed-vs-highest-density-tuning-curve-bands_small}
    \end{minipage}
\end{figure*}

\paragraph{DKW vs. KS vs. LD Bands.}
We recommend the LD bands over the DKW and KS bands (\S\ref{sec:theory:bounding-the-cdf}). While the DKW and KS bands are much better known, both are loose at the extremes due to their constant width over the CDF. In contrast, the LD bands narrow at the extremes, leading to much tighter bounds on the tuning curve for all but the smallest of sample sizes. Figure~\ref{fig:dkw-vs-ks-vs-ld-tuning-curve-bands_small} confirms this hypothesis by comparing the DKW, KS, and LD bands for DeBERTaV3 on MultiNLI. For both small (6) and large (48) samples, the LD bands offer tighter bounds over almost all of the curve.

\paragraph{Equal-tailed vs. Highest Density Bands.}
While \citet{learned-miller-etal-2008-probabilistic} originally constructed CDF bands from equal-tailed intervals, we recommend highest probability density intervals instead. Highest density intervals are more costly to compute, but also yield narrower CDF bands. The narrower CDF bands translate to better tuning curve bands, as shown in Figure~\ref{fig:equal-tailed-vs-highest-density-tuning-curve-bands_small}. In it, we see highest density and equal-tailed LD bands for the tuning curve of DeBERTaV3 on MultiNLI, using both small (6) and large (48) samples. Both bands have similar lower bounds; however, the highest density bands extend further right along the curve, bounding it over a greater range. This effect is even more pronounced in the larger sample.

\subsection{Mean vs. Median Tuning Curves}
\label{sec:analysis:mean-vs-median-tuning-curves}

While prior work focuses on the expected tuning curve, the median ($\M[T_k]$) has several advantages. First, the mean can be difficult to interpret, as we typically do not average over many searches and $T_k$ might have a skewed distribution. Interpreting the median, in contrast, is simple and straightforward: half the time you do better, half the time you do worse. Next, you can only construct nonparametric confidence bands for the mean when $T_k$ is globally bounded. Otherwise, an arbitrarily large number with an arbitrarily small probability could make the mean anything. Finally, even with global bounds, this issue causes the confidence bands for the mean to converge more slowly than those for the median, as shown in Figure~\ref{fig:median-vs-mean}: the median bands converge quickly on the initial part of the curve, while the mean bands remain loose. For all these reasons, we recommend the median over the mean.

\begin{figure}[t]
    \centering
    \includegraphics[width=\linewidth]{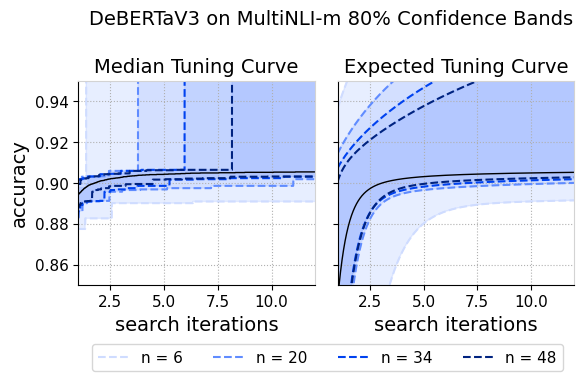}
    \caption{
        Confidence bands for the median vs. mean.
    }
    \label{fig:median-vs-mean}
\end{figure}

\subsection{Effect of Sample Size}
\label{sec:analysis:effect-of-sample-size}

In general, the tuning curve bands' width depends on the CDF's shape. The CDF bands extend a set length above and below the order statistics. Each point on the median bands comes from intersecting a horizontal line with the CDF bands. Thus, the steeper the CDF, the narrower this intersection, the tighter the bands. Intuitively, models that are easy to tune will have tighter median bands because they place more probability near the max, causing the CDF to rise steeply there.

Figure~\ref{fig:effect-of-sample-size} (left) shows how sample size affects this width by plotting DeBERTaV3's tuning curve on MultiNLI across sizes. The main effect extends the range over which the upper bound is non-trivial. Interestingly, increasing the sample size seems to linearly increase this range. Figure~\ref{fig:effect-of-sample-size} (right) plots this relationship, which explains more than 99.9\% of the variance. Thus, \textit{to bound the first $k$ iterations at 80\% confidence, you need about 6.25$k$ samples}.

\begin{figure}[t]
    \centering
    \includegraphics[width=\linewidth]{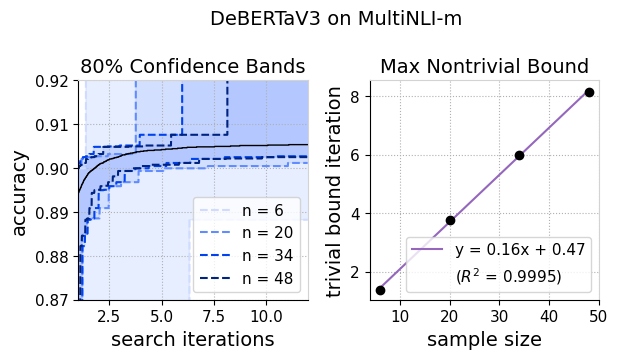}
    \caption{
        (left) Confidence bands for DeBERTaV3's tuning curve on MultiNLI (matched). $n$ represents the sample size, or the number of random search iterations used in constructing the confidence bands. (right) The relationship between the sample size and the search iteration at which the upper confidence band becomes trivial. The relationship is approximately linear, with 99.95\% of the variance explained.
    }
    \label{fig:effect-of-sample-size}
\end{figure}

\section{Related Work}
\label{sec:related-work}

Language models' unprecedented success creates the need for more realistic, comprehensive, and challenging evaluations \citep{ribeiro-etal-2020-beyond, bowman-dahl-2021-will}. As scaling reliably improves performance \citep{hestness-etal-2017-deep, rosenfeld2020a, kaplan-etal-2020-scaling, hernandez-etal-2021-scaling}, performance comparisons without consideration for cost have become inadequate \citep{ethayarajh-jurafsky-2020-utility}: there is always a bigger model with better performance. Thus, many researchers have brought attention to modern models' considerable costs, and proposed frameworks to account for them \citep{strubell-etal-2019-energy, sharir-etal-2020-cost, schwartz-etal-2020-green, henderson-etal-2020-towards}.

Due to hyperparameter search's heavy compute requirements, evaluations should identify not just the best model, but the best model for the tuning budget. Many efficient tuning algorithms have been proposed \citep{bergstra-etal-2012-random, snoek-etal-2012-practical, li-etal-2018-hyperband}; however, these find the best hyperparameters for a given method, rather than evaluate the method for a given budget. Beyond finding the best hyperparameters for \textit{deployment}, researchers and practitioners need systematic tools to manage them during \textit{development}---taking the guesswork out of evaluation. To meet this need, \citet{dodge-etal-2019-show} proposed the first estimator for the expected tuning curve, or mean performance as a function of hyperparameter search iterations.

To compare models fairly, we must fix a tuning algorithm. Random search is easy to implement, a strong baseline, and the basis of several state of the art techniques \citep{li-etal-2018-hyperband, li-etal-2020-system}; thus, it offers a great choice for standardizing tuning curves. Insightfully, \citet{dodge-etal-2019-show} leveraged the independent and identically distributed nature of random search to estimate tuning curves using V-statistics \citep{von-mises-1947-asymptotic}. V-statistics have many desirable properties; however, they can be biased, so \citet{tang-etal-2020-showing} introduced a complementary, unbiased estimator based on U-statistics \citep{hoeffding-1948-class}. Follow up work showed that although these estimators can disagree, neither is uniformly more correct \citep{dodge-etal-2021-expected-validation}. What is more, bootstrapping these estimators fails to create valid confidence bands \citep{tang-etal-2020-showing}. Consequently, it is difficult to know when an estimate is reliable, versus when more data is necessary. We resolve this issue by providing the first valid confidence bands for tuning curves.

\section{Conclusion}
\label{sec:conclusion}

We began with a tutorial on how to tell if a method is cost-effective, well-tuned, or actually better than a baseline, using our confidence bands (\S\ref{sec:tutorial-evaluating-debertav3}). Then, we derived our exact, simultaneous, distribution-free confidence bands for tuning curves (\S\ref{sec:theory}). To complement theory, we designed empirical studies (\S\ref{sec:experimental-setup}). These studies probed the shortcomings of existing solutions (\S\ref{sec:analysis:comparisons-to-existing-methods}), and confirmed our bands achieve the exact coverage necessary to address them (\S\ref{sec:analysis:exact-coverage}). Beyond these main results, ablations revealed how several key ideas tighten the bands (\S\ref{sec:analysis:ablations}). Lastly, we investigated the effect of sample size to show the benefits of the median over the mean tuning curve (\S\ref{sec:analysis:mean-vs-median-tuning-curves}) and illustrate an empirical relationship that informs how many samples are needed to produce useful confidence bands (\S\ref{sec:analysis:effect-of-sample-size}).

Hyperparameters complicate evaluation. Luckily, tuning curves let us compare methods while accounting for tuning effort. Still, point estimates leave open the question of whether or not the data is sufficient to support a given conclusion. To solve this issue, we present the first valid confidence bands for tuning curves. Our confidence bands are simultaneous, distribution-free, and achieve exact coverage in finite samples. Using these confidence bands, researchers and practitioners can compare methods reliably and reproducibly. To analyze your own experiments with confidence, try our library at \url{https://github.com/nicholaslourie/opda}.

\section*{Ethics Statement}
\label{sec:ethics-statement}

We hope our confidence bands will promote more reliable and reproducible work in NLP and related sciences. Confidence bands require more search iterations, thus more compute, than hyperparameter tuning; still, by reducing the frequency of faulty conclusions, we believe our confidence bands will ultimately save resources and drive better science.

\section*{Limitations}
\label{sec:limitations}

While our confidence bands have many strengths, they also have limitations. Mainly, they require the scores are independent and identically distributed (IID), as in random search. We view random search as an ideal tool for research. In development, it can measure tuning difficulty; in production, another algorithm can find optimal hyperparameters for deployment. Even so, a variant of random search, hyperband \citep{li-etal-2018-hyperband}, is competitive with state of the art. With only a small modification, it can satisfy the IID assumption: just fix ahead of time the thresholds for ending training early. Still, while random search excels in many applications, it breaks down if the intrinsic dimension of the search space is too high. Such models will require other techniques to assess tuning difficulty.

While this limitation is shared with prior work \citep{dodge-etal-2019-show, tang-etal-2020-showing, dodge-etal-2021-expected-validation}, our exact coverage guarantee also requires that the \textit{score} distribution is continuous (the \textit{search} distribution can be anything). If it is \textit{not} continuous, then our KS tuning curve bands may be used, as they are conservative for discrete distributions. It is worth investigating whether the LD bands are also conservative for discrete distributions. In general, for any simultaneous confidence bands for the CDF: if the CDF bands are exact, then the median tuning curve bands are exact as well.

All comparisons based on tuning effort share a common limitation: the tuning difficulty depends on the search space. With a small enough search space, a model will always tune faster. Thus, fair comparisons need fair search spaces. Following common practice, we chose each search space with intuition and past experience. While this introduces some subjectivity, it best reflects how people will actually tune the models. Still, those interested in fully objective over pragmatic comparisons might consider the growing literature on \textit{search space learning}, or data-driven methods to design search spaces \citep{perrone-etal-2019-learning, ariafar-etal-2022-predicting}.

\section*{Acknowledgements}
\label{sec:acknowledgements}

This work was supported in part through the NYU IT High Performance Computing resources, services, and staff expertise. This work was supported by Hyundai Motor Company (under the project Uncertainty in Neural Sequence Modeling), the Samsung Advanced Institute of Technology (under the project Next Generation Deep Learning: From Pattern Recognition to AI), and the National Science Foundation (under NSF Award 1922658).

\bibliography{anthology,custom}

\clearpage
\appendix

\section{Proofs}
\label{app:proofs}

If $Y$ has a continuous distribution, then both the KS and LD bands have exact coverage \citep{bradley1968distribution, learned-miller-etal-2008-probabilistic}. We will show if the CDF bands, $\widehat{F}^l(y) \leq F(y) \leq \widehat{F}^u(y)$, have exact coverage then the median tuning curve bands, $\hat{\tau}_m^l(k) \leq \tau_m(k) \leq \hat{\tau}_m^u(k)$, do as well:

\begin{prop}
If $\forall y, \widehat{F}^l(y) \leq F(y) \leq \widehat{F}^u(y)$ with probability $1-\alpha$, then with probability $1-\alpha$, $\forall k, \hat{\tau}_m^l(k) \leq \tau_m(k) \leq \hat{\tau}_m^u(k)$.
\end{prop}

\begin{proof}
We will show the CDF bands on $Y$ hold if and only if the tuning curve bands hold. Since the CDF bands hold with probability $1-\alpha$, the tuning curve bands will then also hold with probability $1-\alpha$.

Recall how to construct the median tuning curve confidence bands. First, take some simultaneous confidence bands for the CDF of $Y$:
$$ \widehat{F}^l(y) \leq F(y) \leq \widehat{F}^u(y) $$
then, translate them into simultaneous confidence bands for the CDF of $T_k$ using Equation~\ref{eq:cdf-of-max}:
\begin{align*}
    \widehat{F}^l(y)^k &\leq F(y)^k \leq \widehat{F}^u(y)^k \\
    \widehat{F}^l(y)^k &\leq F_k(y) \leq \widehat{F}^u(y)^k
\end{align*}
Thus, the CDF bands on $Y$ hold if and only if the CDF bands on $T_k$ hold.

Finally, we take the median of the upper band to get a lower bound, and the median of the lower band to get an upper bound:
\begin{align*}
    \hat{\tau}_m^l(k) &\coloneqq \min \left\{Y_{(i)} \mid 0.5 \leq \widehat{F}^u\left(Y_{(i)}\right)^k \right\} \\
    \hat{\tau}_m^u(k) &\coloneqq \min \left\{Y_{(i)} \mid 0.5 \leq \widehat{F}^l\left(Y_{(i)}\right)^k \right\}
\end{align*}
Geometrically, this corresponds to drawing a horizontal line at 0.5 probability across the CDF plot, and finding where it intersects the confidence bands. Note that we only need to check the CDF bands at the order statistics since they are step functions that only change at those points.

Assume the CDF bands on $Y$ hold, then we have $\forall y, \widehat{F}^l(y)^k \leq F_k(y)$, so:
\begin{equation*}
    \begin{split}
        \Big\{Y_{(i)} \mid& 0.5 \leq \widehat{F}^l\left(Y_{(i)}\right)^k \Big\} \\
        &\subseteq \Big\{Y_{(i)} \mid 0.5 \leq F_k\left(Y_{(i)}\right) \Big\}
    \end{split}
\end{equation*}
Therefore:
\begin{equation*}
    \begin{split}
        \min \Big\{Y_{(i)} \mid& 0.5 \leq \widehat{F}^l\left(Y_{(i)}\right)^k \Big\} \\
        &\geq \min \Big\{Y_{(i)} \mid 0.5 \leq F_k\left(Y_{(i)}\right) \Big\}
    \end{split}
\end{equation*}
Thus:
$$ \hat{\tau}_m^u(k) \geq \M(T_k) $$
The case for $\hat{\tau}_m^l \leq \M(T_k)$ is analogous. So, if the CDF bands hold, then the tuning curve bands hold.

For the reverse implication, assume that the CDF bands on $Y$ do not hold, then there exists some $y$ such that either $\widehat{F}^l(y) > F(y)$ or $F(y) > \widehat{F}^u(y)$. We will show the case for $\widehat{F}^l(y) > F(y)$. First, let $k \coloneqq \log_{F(y)}(0.5)$, so in particular $F(y)^k = 0.5$, thus $\M[T_k] = y$. We have:
\begin{equation*}
    \widehat{F}^l(y)^k > F(y)^k = F(y)^{\log_{F(y)}(0.5)} = 0.5
\end{equation*}
Consider the order statistic, $Y_{(j)}$, immediately preceding $y$. Since the score distribution is continuous, with probability 1 we have $y \not= Y_{(j)}$ and thus $Y_{(j)} < y$. Because $\widehat{F}^l(y)$ is a step function that only changes at the order statistics, and $Y_{(j)}$ immediately precedes $y$, we must have:
\begin{equation*}
    \widehat{F}^l\left(Y_{(j)}\right)^k = \widehat{F}^l(y)^k > 0.5
\end{equation*}
Thus, $Y_{(j)} \in \left\{Y_{(i)} \mid 0.5 \leq \widehat{F}^l\left(Y_{(i)}\right)^k \right\}$, so:
\begin{align*}
    \hat{\tau}_m^u(k) &= \min \left\{Y_{(i)} \mid 0.5 \leq \widehat{F}^l\left(Y_{(i)}\right)^k \right\} \\
                      &\leq Y_{(j)} \\
                      &< y \\
                      &= \M[T_k]
\end{align*}
So, $\hat{\tau}_m^u(k) < \M[T_k]$ so the tuning curve confidence bands are violated. The other case looks similar.

Since the CDF bands for $Y$ hold if and only if the tuning curve bands hold, they hold with the same probability, $1-\alpha$.
\end{proof}

Next, we will show that the mean tuning curve confidence bands are conservative.

\begin{prop}
If $\forall y, \widehat{F}^l(y) \leq F(y) \leq \widehat{F}^u(y)$ with probability $1-\alpha$, then with probability greater than or equal to $1-\alpha$, $\forall k, \hat{\tau}_e^l(k) \leq \tau_e(k) \leq \hat{\tau}_e^u(k)$.
\end{prop}

\begin{proof}
Given two CDFs, $F(y)$ and $G(y)$, let $F(y) \preceq G(y)$ denote that the distribution for $F(y)$ is less than or equal to $G(y)$ in the usual stochastic order (i.e., first-order stochastic dominance).

If $\forall y, \widehat{F}^l(y) \leq F(y) \leq \widehat{F}^u(y)$, then we have that $\forall y, \widehat{F}^l(y)^k \leq F_k(y) \leq \widehat{F}^u(y)^k$. This fact is equivalent to $\widehat{F}^u(y)^k \preceq F_k(y) \preceq \widehat{F}^l(y)^k$. It is then a standard fact that this implies the expectation of $\widehat{F}^l(y)^k$ is greater than the expectation of $F_k(y)$ which is greater than the expectation of $\widehat{F}^u(y)^k$. In other words:  $\hat{\tau}_e^l(k) \leq \tau_e(k) \leq \hat{\tau}_e^u(k)$.

Because $\forall y, \widehat{F}^l(y) \leq F(y) \leq \widehat{F}^u(y)$ implies $\forall k, \hat{\tau}_e^l(k) \leq \tau_e(k) \leq \hat{\tau}_e^u(k)$, the latter statement must hold with at least the probability of the former. Thus, it holds with probability at least $1-\alpha$.
\end{proof}

Unlike the median tuning curve bands, the mean tuning curve bands will be strictly conservative in general. One reason for this is that the unbounded probability mass will float to the bounds of the distribution's support in the worst case. Another reason is that the CDF of $T_k$ could briefly violate its confidence bounds but still end up with a lower (or higher) mean than the upper (or lower) confidence band. Intuitively, this issue comes from the fact that unlike a quantile, which involves a single point of the CDF, the mean involves the entire shape of the CDF. Exact confidence bounds for the expected tuning curve likely require stronger assumptions.

\section{Experimental Setup}
\label{app:experimental-setup}

Expanding on \S\ref{sec:experimental-setup}, this appendix documents the full experimental details. Our code for the confidence bands and analysis is available at \url{https://github.com/nicholaslourie/opda} (tag: v0.6.1).

\subsection{Comparisons to Prior Work}
\label{app:experimental-setup:comparisons-to-prior-work}

In our comparisons to prior work, we use the data from \citet{tang-etal-2020-showing} located at \url{github.com/castorini/meanmax} (commit: 0ea1241), which is MIT licensed. The data contains hyperparameter searches for two different models: an MLP and an LSTM on the Reuters document classification dataset \citep{apte-etal-1994-automated}. Search was run for 145 iterations on the MLP and 152 iterations on the LSTM, using F1 score as the validation metric.

For the MLP, the search distribution was uniform over \{16, 32, 64\} for the batch size, a learning rate of 0.001, discrete uniform over 0-$10^7$ for the random seed, uniform over $[0.05, 0.7]$ for the dropout rate, 1 hidden layer, and discrete uniform over 256-768 for the hidden dimension.

The LSTM had a nonstandard hyperparameter. For it, "static" initialized word embeddings with frozen word2vec vectors \citep{mikolov2013distributed}, "nonstatic" initialized word embeddings with trainable word2vec vectors, and "rand" initialized word embeddings randomly. For the LSTM, the search distribution was static, nonstatic, or rand with a 40\%, 50\%, and 10\% chance, uniform over \{16, 32, 64\} for batch size, a truncated exponential over $[0.001, 0.099]$ for learning rate, a discrete uniform over 0-$10^7$ for the random seed, 1 or 2 layers with a 75\% and 25\% chance, discrete uniform over 384-768 for hidden dimension, a uniform over $[0, 0.3]$ for weight dropout rate, a uniform over $[0, 0.3]$ for embedding dropout rate, and a uniform over $[0.985, 0.995]$ for the coefficient to use in exponentially averaging the parameters.

For our kernel density estimates (KDEs), we used a Gaussian kernel with reflection about the support's boundary (0 and 1 for F1 score) as a boundary correction \citep{jones1993simple}. We selected the bandwidth from the grid: 1.00\mbox{e-1}, 5.00\mbox{e-2}, 2.50\mbox{e-2}, 1.25\mbox{e-2}, and 6.25\mbox{e-3}, by visually inspecting the PDF and CDF plots for the resulting KDE. Ultimately, we chose 5.00\mbox{e-2} and 1.25\mbox{e-2} for the LSTM and MLP KDE bandwidths, respectively.

In our simulation for computing the bootstrap confidence bands' coverage, we ran 4,096 rounds, in each round testing whether or not the bootstrap confidence bands covered the true tuning curve at each point. To construct the bootstrap confidence bands, we sampled 50 points from the simulated score distribution and then resampled this initial sample 4,096 times to determine the bootstrap distribution's quantiles for the different estimators at each point of the tuning curve.

\subsection{Analysis in a Realistic NLP Scenario}
\label{app:experimental-setup:analysis-in-a-realistic-nlp-scenario}

For our analysis of a realistic NLP scenario, we train DeBERTa \citep{he2021deberta} and DeBERTaV3 \citep{he2021debertav3} using the original DeBERTa codebase: \url{github.com/microsoft/DeBERTa} (commit: c558ad9), which is MIT licensed. We fine-tuned base model sizes on MultiNLI \citep{williams-etal-2018-broad} and evaluated using accuracy. For both models, we ran 1,024 rounds of random search and used the same search distribution for each.

The search distribution was discrete uniform over 16-64 for batch size, discrete uniform over 1-4 for number of epochs, uniform over $[0, 0.6]$ for warmup proportion of the first epoch, log uniform over $[10^{-6}, 10^{-3}]$ for learning rate, and uniform over $[0, 0.3]$ for dropout rate. For ease of comparison, we used the same sample of 1,024 hyperparameters for both models. For other implementation details, we used the fp16 floating point format, a maximum sequence length of 256, an evaluation batch size of 256, and we logged progress every 1000 steps. All other hyperparameters were identical to DeBERTa and DeBERTaV3's defaults for MultiNLI.

Each training run for DeBERTa or DeBERTaV3 on a given set of hyperparameters was run via SLURM with a single RTX 8000 NVIDIA GPU (48 GB), 2 CPU cores from an Intel Xeon Platinum 8268 Processor (2.90 GHz), and 32 GB of CPU memory. For software, the code was run on Ubuntu 18.04.5 LTS with CUDA V10.1.243, Conda 32.1.0, Python 3.7.16, and PyTorch 1.7.0.

Hyperparameter tuning jobs were run in parallel. For DeBERTa-base, the jobs took an average of 4h 57m 57s with a standard deviation of 2h 19m 0s, a minimum of 1h 43m 28s, and a maximum of 11h 9m 8s. For DeBERTaV3-base, the jobs took an average of 4h 8m 30s with a standard deviation of 1h 56m 50s, a minimum of 1h 24m 57s, and a maximum of 9h 29m 38s.

For the tutorial on evaluating DeBERTaV3 (\S\ref{sec:tutorial-evaluating-debertav3}), we subsampled 48 hyperparameter evaluations without replacement from the total 1,024 for both DeBERTa and DeBERTaV3 in order to estimate the tuning curves with confidence bands.

For our kernel density estimates (KDEs), we used a Gaussian kernel. Due to the large sample sizes (1,024 samples for both DeBERTa and DeBERTaV3), the amount of probability mass outside the support's bounds (0 and 1 for accuracy) was smaller than the numerical precision; therefore, we did not need to apply a boundary correction to the KDE. We selected the bandwidth from the grid: 1.0\mbox{e-1}, 7.5\mbox{e-2}, 5\mbox{e-2}, and 2.5\mbox{e-2}, by visually inspecting the PDF and CDF plots for the resulting KDE. Ultimately, we chose 5\mbox{e-2} as the bandwidth for both DeBERTa and DeBERTaV3's simulations.

To estimate the median tuning curve confidence bands' empirical coverage, we ran 1,024 rounds of simulation. In each round, we sampled 48 scores from the kernel density estimate, constructed the median tuning curve confidence bands, and then measured how often the entire median tuning curve was covered simultaneously. Because the CDF and median tuning curve are increasing and their confidence bands are step functions, it is only necessary to check for violations around the discontinuities.

\section{Additional Algorithms}
\label{app:additional-algorithms}

In \S\ref{sec:theory:bounding-the-cdf}, Equation~\ref{eq:l-statistic} defines the test statistic, $L_n$:
$$ L_n \coloneqq \max_{i=1\ldots n} B_{i}\left(F\left(Y_{(i)}\right)\right) $$
Where $B_{i}\left(p\right)$ is coverage under $\text{Beta}(i, n{+}1{-}i)$ of the smallest (equal-tailed or highest density) interval containing $p$. To compute $L_n$ or construct confidence bands for a given value of $L_n$, we then need the ability to compute (equal-tailed or highest density) intervals for the beta distribution, and to compute the coverage of the smallest such interval containing a point $p$. We describe algorithms for these tasks here. For full implementation details, see the code available at \url{https://github.com/nicholaslourie/opda} (tag: v0.6.1).

Recall that when $a > 1$ or $b > 1$, $\text{Beta}(a, b)$ is unimodal, making the highest density interval well-defined. If $a, b \geq 1$, then $p^* = \frac{a - 1}{a + b - 2}$ is the mode. We only consider the distributions $\text{Beta}(i, n{+}1{-}i)$, thus whenever we have two or more samples the distributions will all be unimodal. Many statistical packages make available the beta's density, $g(p)$, CDF, $G(p)$, and inverse CDF or quantile function, $G^{-1}(u)$, so we assume access to these functions.

\subsection{Intervals}
\label{app:additional-algorithms:intervals}

Given a $\text{Beta}(a, b)$ distribution, we must compute intervals containing $1-\alpha$ of the probability mass.

The \textit{equal-tailed interval} puts equal probability in the distribution's tails above and below it. It may be computed directly from the inverse CDF:
$$ \left[G^{-1}\left(\frac{\alpha}{2}\right), G^{-1}\left(1 - \frac{\alpha}{2}\right)\right] $$
The \textit{highest density interval} circumscribes a region where the probability density is highest. It is more challenging to compute. Given a lower end point, $p_l$, we can construct the interval with coverage $1-\alpha$ by setting the upper end point, $p_u$, to:
$$ p_u = G^{-1}\left(G\left(p_l\right) + 1 - \alpha\right) $$
This is the highest density interval with coverage $1-\alpha$ precisely when the density is equal at the end points:
$$ g(p_l) = g(p_u) $$
Thus, we can use any root finding algorithm to identify the value of $p_l$ for which the end points have equal density. In other words, we solve the following equation for $p_l$:
$$ g(p_l) - g\left(G^{-1}\left(G\left(p_l\right) + 1 - \alpha\right)\right) = 0 $$
Our implementation uses the bisection method (i.e., binary search) due to its simplicity and robustness. Letting $p^* = \frac{a - 1}{a + b - 2}$ (the mode), we initialize the lower bound for $p_l$ to:
$$ G^{-1}\left(\max \left\{G\left(p^*\right) - (1 - \alpha), 0\right\}\right) $$
and we initialize the upper bound for $p_l$ to:
$$ \min \left\{ p^*, G^{-1}\left(\alpha\right) \right\} $$
We then use binary search to shrink the interval to within some tolerance.

\subsection{Coverage}
\label{app:additional-algorithms:coverage}

To compute the $L_n$ statistic, we need to compute the coverage under $\text{Beta}(a, b)$ of the smallest equal-tailed or highest density interval containing $p$.

Given $p$, the equal-tailed interval's coverage is:
$$ 2 \left|\frac{1}{2} - G(p)\right| $$
The highest density interval requires finding a root. The smallest highest density interval containing $p$ either has $p$ as the lower or the upper endpoint. We can identify which it is by comparing $p$ with the mode, $p^*$. Consider when $p = p_l$, or $p$ is the lower endpoint. The upper end, $p_u$, must have the same density as the lower, so it satisfies the equation:
$$ g\left(p\right) - g\left(p_u\right) = 0 $$
We could solve this equation using any root finding algorithm. Our implementation uses the bisection method (binary search), initializing the bounds to $0$ and $p^*$ when seeking the lower end, and $p^*$ and $1$ when seeking the upper end.

\begin{figure}[t]
    \centering
    \includegraphics[width=\linewidth]{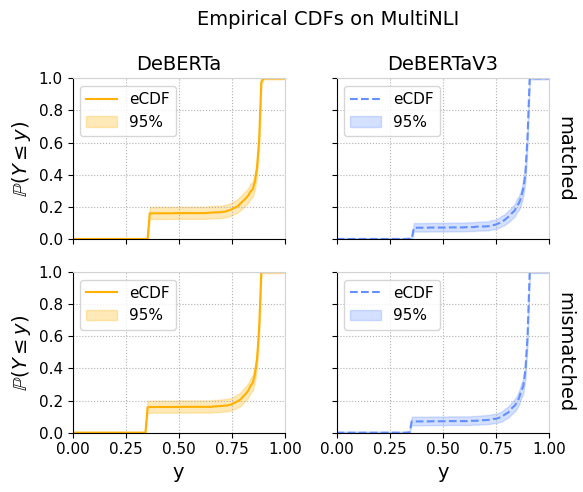}
    \caption{
        Highest density LD bands for the samples of 1,024 search iterations described in \S\ref{sec:experimental-setup} and \S\ref{app:experimental-setup:analysis-in-a-realistic-nlp-scenario}.
    }
    \label{fig:deberta-and-debertav3-cdf-fit}
\end{figure}

We experimented with a number of root finding algorithms (the bisection method, secant method, regula falsi, Newton's method, hybridizations); however, a good implementation of the bisection method proved fastest and most practical due to its reliability. Because we need to find roots for many different beta distributions when constructing the confidence bands, the faster convergence of Newton's method failed to make up for the cost of recovering from failures to converge when performing the algorithm in a vectorized way, though perhaps better implementations could improve upon this.

\begin{figure}[t]
    \centering
    \includegraphics[width=\linewidth]{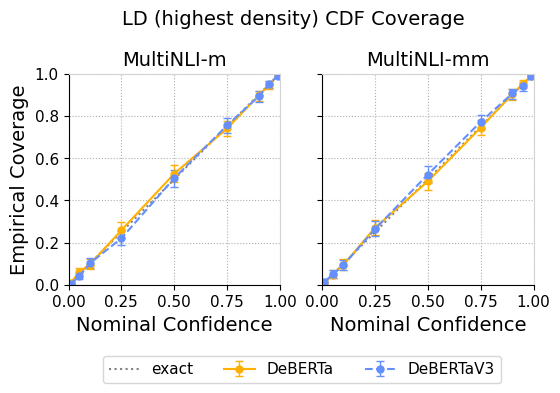}
    \caption{
        Nominal vs. empirical coverage of the LD CDF bands for the score distribution. Error bars show 99\% Clopper-Pearson confidence intervals. The results lie on the $y=x$ line, indicating exact coverage.
    }
    \label{fig:exact-cdf-coverage}
\end{figure}

\section{Exact Coverage}
\label{app:exact-coverage}

In \S\ref{sec:analysis:exact-coverage}, we simulate a practical hyperparameter tuning scenario by first running random search on DeBERTa and DeBERTaV3 1,024 times each, then using the resulting data for kernel density estimates of the true distributions. To realistically simulate these scenarios, we need a sample size that is large enough to guarantee a close approximation to the underlying distributions. Our samples do in fact achieve such approximations, as shown by the simultaneous 95\% confidence bands for the CDFs presented in Figure~\ref{fig:deberta-and-debertav3-cdf-fit}. Since the confidence bands are so narrow, we can conclude the samples' eCDFs adhere tightly to the true CDFs.

In addition to the LD tuning curve confidence bands, we empirically validated the LD CDF bands' coverage. Figure~\ref{fig:exact-cdf-coverage} confirms the LD CDF bands, and in particular our implementation via inverting the hypothesis test, achieve exact coverage.

\section{Ablation Studies}
\label{app:ablation-studies}

Expanding on \S\ref{sec:analysis:ablations}, we provide extended ablation studies. In general, the conclusions mirror those of \S\ref{sec:analysis:ablations}. Figure~\ref{fig:dkw-vs-ks-vs-ld-cdf-bands} and Figure~\ref{fig:equal-tailed-vs-highest-density-cdf-bands} introduce new results respectively by comparing the CDF bands (as opposed to tuning curves) from the DKW, KS, and LD methods and the equal-tailed and highest density intervals. Figure~\ref{fig:dkw-vs-ks-vs-ld-tuning-curve-bands} and Figure~\ref{fig:equal-tailed-vs-highest-density-tuning-curve-bands} expand on the existing tuning curve results by comparing the DKW, KS, and LD methods and the equal-tailed and highest density intervals for more sample sizes. For even more results, see the Jupyter notebooks at \url{https://github.com/nicholaslourie/opda} (tag: v0.6.1).

\begin{figure*}
    \centering
    \includegraphics[width=0.95\textwidth]{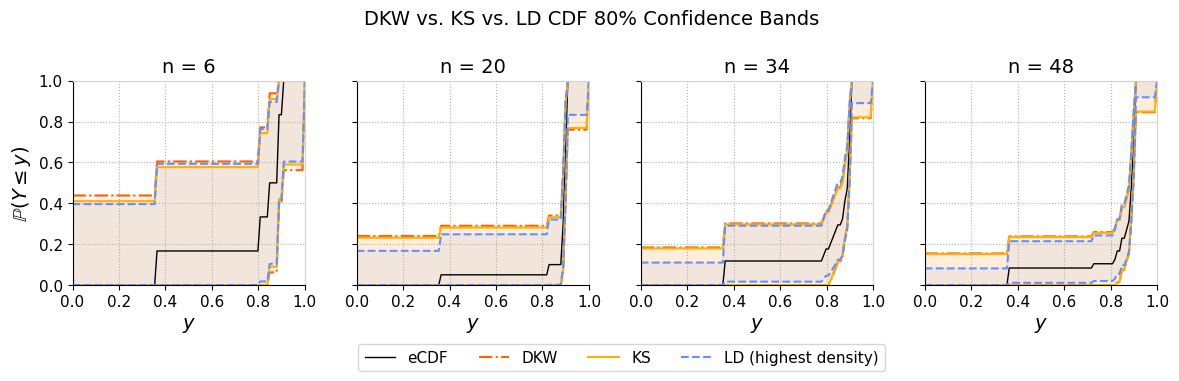}
    \caption{
        DKW, KS, and LD confidence bands for DeBERTaV3's score distribution on MultiNLI (matched).
    }
    \label{fig:dkw-vs-ks-vs-ld-cdf-bands}
\end{figure*}

\begin{figure*}
    \centering
    \includegraphics[width=0.95\textwidth]{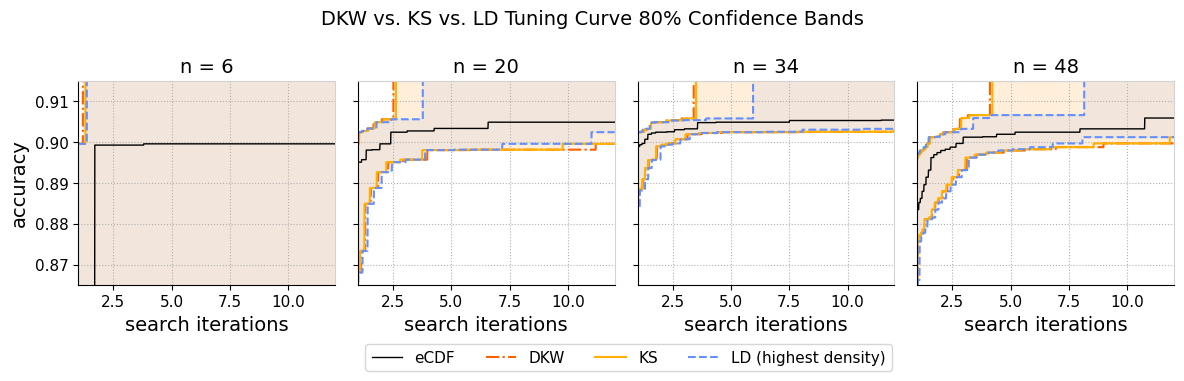}
    \caption{
        DKW, KS, and LD confidence bands for DeBERTaV3's tuning curve on MultiNLI (matched).
    }
    \label{fig:dkw-vs-ks-vs-ld-tuning-curve-bands}
\end{figure*}

\begin{figure*}
    \centering
    \includegraphics[width=0.95\textwidth]{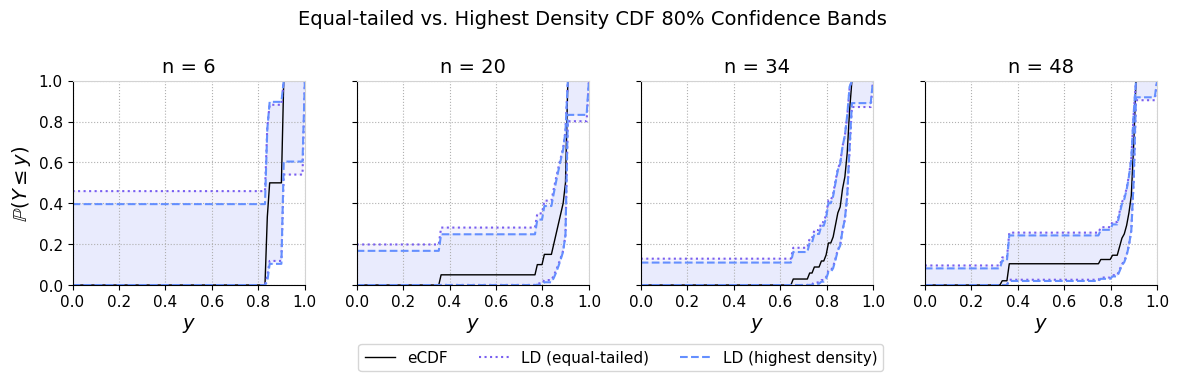}
    \caption{
        Equal-tailed and highest density LD bands for DeBERTaV3's score distribution on MultiNLI (matched).
    }
    \label{fig:equal-tailed-vs-highest-density-cdf-bands}
\end{figure*}

\begin{figure*}
    \centering
    \includegraphics[width=0.95\textwidth]{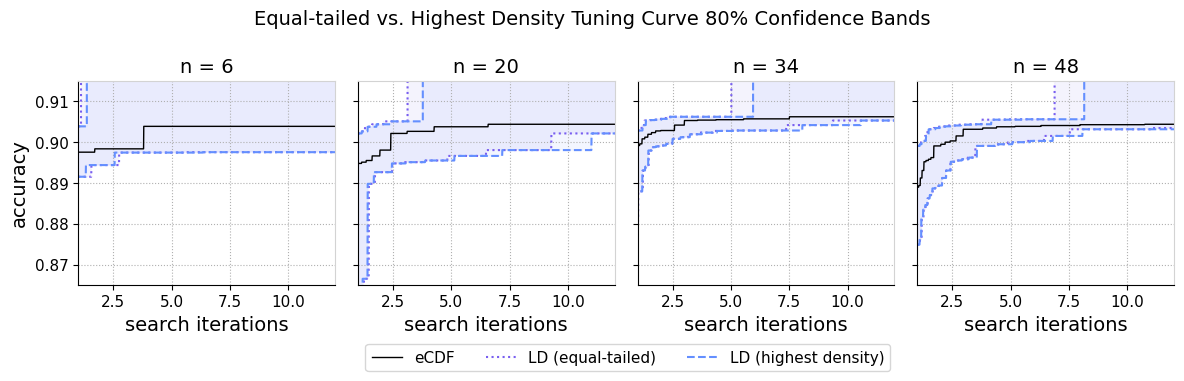}
    \caption{
        Equal-tailed and highest density LD bands for DeBERTaV3's tuning curve on MultiNLI (matched).
    }
    \label{fig:equal-tailed-vs-highest-density-tuning-curve-bands}
\end{figure*}

\section{Extended Results}
\label{app:extended-results}

In our experiments, we used between 6 and 48 iterations of random search to keep things realistic for research and practical applications. At the same time, we ran many more iterations to make the exact coverage analysis as rigorous as possible: 1,024 for both DeBERTa and DeBERTaV3. Since these large samples are already available, we now run extended analyses to satisfy our own curiosity and that of the interested reader.

One question is: what do turning curves look like in the large sample limit? Figure~\ref{fig:debertav3-tuning-curve_large-sample} answers this for DeBERTaV3. It is similar to Figure~\ref{fig:debertav3-tuning-curve} but using all 1,024 search iterations. Continuing in this way, Figure~\ref{fig:deberta-vs-debertav3_large-sample} is similar to Figure~\ref{fig:deberta-vs-debertav3}, but again using all 1,024 iterations of random search. The results remain largely the same. High performance is achieved after 6 or so iterations, and it more or less saturates after 12 to 24---though performance continues to rise very slowly up into the hundreds.

We can also revisit our hyperparameter analysis with larger samples. Consider the hyperparameter importance. Figure~\ref{fig:epoch-importance_large-sample} uses 245 search iterations to compare tuning the epochs against fixing it at the default. The tighter bands agree with our previous findings. There is weak evidence that the default beats tuning epochs initially, but the difference is small. The larger sample size also provides enough statistical power for a finer-grained analysis of the epochs' \textit{effect}. By fixing epochs to various values, we can visualize its impact on the tuning curve. In this way, Figure~\ref{fig:epoch-effect_large-sample} shows that 1 or 2 epochs is too low, while 3 or 4 epochs appears to be enough for fine-tuning DeBERTaV3 on MultiNLI.

\begin{figure}[h]
    \centering
    \includegraphics[width=\linewidth]{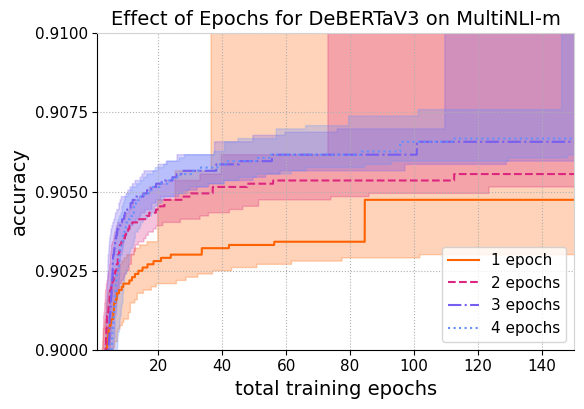}
    \caption{
        Median tuning curves for DeBERTaV3 on MultiNLI (matched), with 80\% confidence based on 245 search iterations each. Fixing epochs to different values visualizes the hyperparameter's effect.
    }
    \label{fig:epoch-effect_large-sample}
\end{figure}

\begin{figure}[h]
    \centering
    \includegraphics[width=\linewidth]{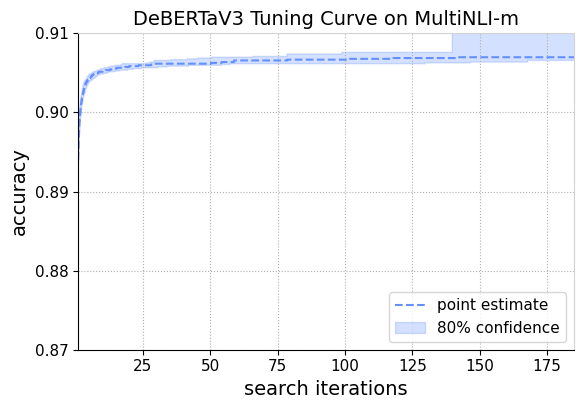}
    \caption{
        The median tuning curve for DeBERTaV3 on MultiNLI (matched), based on 1,024 search iterations.
    }
    \label{fig:debertav3-tuning-curve_large-sample}
\end{figure}

\begin{figure}[h]
    \centering
    \includegraphics[width=\linewidth]{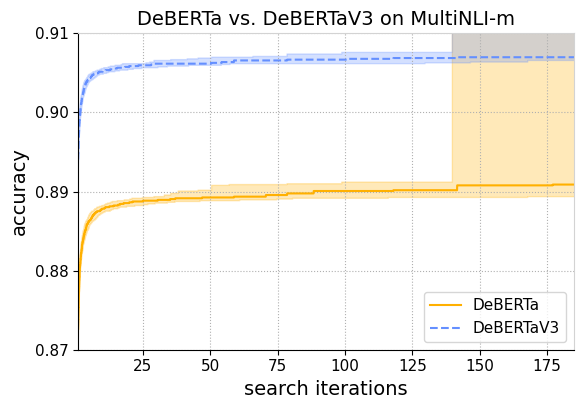}
    \caption{
        Median tuning curves for DeBERTaV3 and DeBERTa on MultiNLI (matched), with 80\% confidence based on 1,024 search iterations.
    }
    \label{fig:deberta-vs-debertav3_large-sample}
\end{figure}

\begin{figure}[h]
    \centering
    \includegraphics[width=\linewidth]{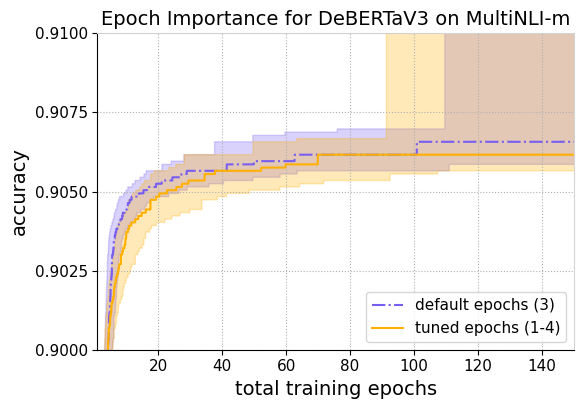}
    \caption{
        Median tuning curves for DeBERTaV3 on MultiNLI (matched), with 80\% confidence based on 245 search iterations. To assess hyperparameter importance, the curves compare tuning epochs vs. using the default.
    }
    \label{fig:epoch-importance_large-sample}
\end{figure}

\end{document}